\documentclass[11pt]{article}

\usepackage[margin=1in]{geometry}
\usepackage{amsthm}
\usepackage{amssymb}
\usepackage{amsmath}
\usepackage{graphicx}
\usepackage{amsfonts}
\usepackage{url}
\usepackage{color}

\newtheorem{theorem}{Theorem}[section]

\newtheorem{lemma}[theorem]{Lemma}
\newtheorem{proposition}[theorem]{Proposition}

\theoremstyle{definition}

\usepackage{natbib}
\bibpunct{(}{)}{;}{a}{,}{,}

\title{Monte Carlo approximation certificates for $k$-means clustering}

\author{Dustin~G.~Mixon\footnote{Department of Mathematics, The Ohio State University, Columbus, OH.} \qquad Soledad Villar\footnote{Center for Data Science, New York University, New York, NY and Simons Institute for Computing, UC Berkeley, Berkeley, CA. Correspondence to: \texttt{soledad.villar@nyu.edu}. 60 5th Ave, office 621. New York, NY 10011}}
\date{}

\providecommand{\keywords}[1]{\textbf{{Keywords}} #1}

\begin{document}
\maketitle

\begin{abstract}
Efficient algorithms for $k$-means clustering frequently converge to suboptimal partitions, and given a partition, it is difficult to detect $k$-means optimality.
In this paper, we develop an \emph{a posteriori} certifier of approximate optimality for $k$-means clustering.
The certifier is a sub-linear Monte Carlo algorithm based on Peng and Wei's semidefinite relaxation of $k$-means.
In particular, solving the relaxation for small random samples of the data set produces a high-confidence lower bound on the $k$-means objective, and being sub-linear, our algorithm is faster than $k$-means++ when the number of data points is large.
We illustrate the performance of our algorithm with both numerical experiments and a performance guarantee:
If the data points are drawn independently from any mixture of two Gaussians over $\mathbb{R}^m$ with identity covariance, then with probability $1-O(1/m)$, our $\operatorname{poly}(m)$-time algorithm produces a 3-approximation certificate with 99\% confidence.
\end{abstract}

\keywords{k-means clustering, semidefinite programming, approximation certificate, Gaussian mixture model, random matrices}

\section{Introduction}

Given a set of points in a metric space, the clustering problem asks for a partition of the set that minimizes a given dissimilarity objective function. 
When points lie in Euclidean space, the $k$-means objective is a popular choice for this dissimilarity function.
Unfortunately, the resulting $k$-means clustering problem is NP-hard~(see \citet{AloiseD:09, Awasthi:15}).
Despite its worst-case hardness, many algorithmic approaches have proven to be good enough for practical applications (e.g., $k$-means++ by~\citet{ArthurV:07}).
Since efficient algorithms for $k$-means frequently converge to local optima, we are interested in computing certificates of (approximate) optimality for clustering solutions. 

To this end, $k$-means++ enjoys an approximation guarantee for the average case, which can be converted to a Monte Carlo approximation certificate (see the next section for details).
In particular, there is a lower bound on the $k$-means value that can be expressed as the mean of a nonnegative random variable, and since instances of this random variable are easy to compute, they can be used to form a test statistic that refutes a null hypothesis in favor of a lower bound.
Unfortunately, the resulting lower bound is quite small in practice.

Recently, the authors developed a \textit{probably certifiably correct} (PCC) algorithm (\citet{Bandeira:16}) for $k$-means clustering (\citet{IguchiMPV:15}). \nocite{AwasthiB:15}
If the data satisfies certain hypotheses, then given the optimal partition, the PCC algorithm produces a certificate of optimality in quasi-linear time.
Unfortunately, these hypotheses require the semidefinite relaxation of the $k$-means problem from \citet{PengW:07} to be tight, which many real-world data sets fail to satisfy.
When the Peng--Wei relaxation is not tight, we can still round the solution of the $k$-means semidefinite program (SDP) as in~\citet*{MixonVW:17}, but this is extremely slow compared to $k$-means++.

Amazingly, while the $k$-means SDP is slow, we can still use it to produce a competitive sub-linear Monte Carlo approximation certificate.
In particular, running the SDP on a random sample of the data yields a nonnegative random value whose expectation is a lower bound on the $k$-means value of the complete data set.
In practice, the resulting \textit{a posteriori} lower bound is a substantial improvement over the $k$-means++ guarantee, and furthermore, when the data set is large, we can compute our lower bound faster than $k$-means++ can cluster.

In this paper, we first describe the natural Monte Carlo implementation of the $k$-means++ guarantee.
We then apply similar ideas to the $k$-means SDP of a random sample of the data set.
Next, we prove that our method provides near-optimal bounds for $k$-means clustering when the data is drawn from a mixture of Gaussians.
Figure~\ref{fig.lower_bound} presents numerical experiments that illustrate the performance of our method on real-world data. 
We conclude with a discussion of possible ways to improve our results in future work.

\section{Method}

Given $\{x_i\}_{i\in T}\subseteq\mathbb{R}^m$, the we seek the partition of $T$ that minimizes the $k$-means objective:
\begin{equation}
\tag{$T$-IP}
\label{eq.T-IP}
\text{minimize}
\quad
\frac{1}{|T|}\sum_{t\in[k]}\sum_{i\in C_t}\bigg\|x_i-\frac{1}{|C_t|}\sum_{j\in C_t}x_j\bigg\|^2
\quad
\text{subject to}
\quad
C_1\sqcup\cdots\sqcup C_k=T
\end{equation}
A common heuristic for this optimization problem is Lloyd's algorithm, also known as the $k$-means algorithm, which alternates between computing centroids of the current partition and re-assigning points to the nearest centroid.
This method is very fast, each iteration costing only $O(kmN)$ operations, where $N=|T|$.
Moreover, one may initialize with random proto-centroids in such a way that the corresponding partition of $T$ has random $k$-means value $W$ such that
\begin{equation}
\label{eq.kmeans++}
\operatorname{val}(T\text{-}\mathrm{IP})
\geq\frac{1}{8(\log k+2)}\cdot\mathbb{E}W.
\end{equation}
Lloyd's algorithm with this initialization is known as $k$-means++ (\citet{ArthurV:07}), and since the $k$-means value monotonically decreases with each iteration of Lloyd's algorithm, one may conclude that $k$-means++ is $O(\log k)$-competitive with the optimal $k$-means solution on average.

The present paper is concerned with quickly computing a lower bound on the optimal $k$-means value of a given problem instance.
To this end, one may apply \eqref{eq.kmeans++} to produce a Monte Carlo approximation certificate.
Specifically, observe that $V:=W/(8(\log k+2))$ satisfies three conditions:
\begin{itemize}
\item[(i)]
realizations of $V$ can be computed efficiently,
\item[(ii)]
$V\geq0$ almost surely, and
\item[(iii)]
$\operatorname{val}(T\text{-}\mathrm{IP})\geq\mathbb{E}V$.
\end{itemize}
While we do not have direct access to $\mathbb{E}V$, we may leverage these three features of $V$ to quickly conclude with statistical confidence that $\operatorname{val}(T\text{-}\mathrm{IP})>B$ for some $B$.

Specifically, we may perform a hypothesis test
\begin{align*}
&H_0:\mathbb{E}V\leq B\\
&H_1:\mathbb{E}V>B
\end{align*}
by drawing independent instances $V_1,\ldots V_\ell$ of $V$ and computing a test statistic:
\begin{equation}
\label{eq.test statistic}
T=\min_{i\in[\ell]}V_i.
\end{equation}
This test statistic can be computed efficiently due to (i).
Assuming $H_0$, then
\[
\mathbb{P}(T\geq t)
=\Big(\mathbb{P}(V_i\geq t)\Big)^\ell
\leq(\mathbb{E}V/t)^\ell
\leq(B/t)^\ell,
\]
where the first inequality applies Markov's inequality with (ii), whereas the final inequality is due to $H_0$.
As such, given $T$, we may reject $H_0$ with $p$-value $\eta=(B/T)^\ell$, i.e., with confidence $1-\eta$.
This in turn implies $\operatorname{val}(T\text{-}\mathrm{IP})>B$ by (iii). 

Unfortunately, applying this procedure with $V=W/(8(\log k+2))$ will often produce a loose lower bound on the output of $k$-means++.
For example, running $k$-means++ on the MNIST training set of 60,000 handwritten digits from~\citet{LeCunCB:online} with $k=10$ produces $k$-means values over 39, whereas $V$ tends to lie below 3.
How can we get a better \textit{a posteriori} approximation certificate?

In theory, we could run the Peng--Wei semidefinite relaxation:
\begin{equation}
\tag{$T$-SDP}
\label{eq.S-SDP}
\text{minimize}
\quad
\frac{1}{2|T|}\operatorname{tr}(D_TX)
\quad
\text{subject to}
\quad
X1=1,~\operatorname{tr}(X)=k,~X\geq0,~X\succeq0
\end{equation}
where $D_T$ denotes the $T\times T$ matrix whose $(i,j)$th entry is $\|x_i-x_j\|^2$.
Indeed, for each partition $C_1\sqcup\cdots\sqcup C_k=T$, the matrix $X=\sum_{t=1}^k\frac{1}{|C_t|}1_{C_t}1_{C_t}^\top$ is feasible in this semidefinite program with value equal to the partition's $k$-means value.
As such, $\operatorname{val}(T\text{-}\mathrm{IP})\geq\operatorname{val}(T\text{-}\mathrm{SDP})$, but computing $\operatorname{val}(T\text{-}\mathrm{SDP})$ is computationally prohibitive unless the data set is sufficiently small (say, $\leq1,000$).
Recently, the authors established that, if the clusters are sufficiently well behaved, then this SDP relaxation is tight, and furthermore, tightness can be established quickly from the optimal clustering by constructing a dual certificate (\citet{IguchiMPV:15}).
Sadly, many real-world data sets (such as MNIST) do not exhibit such nice behavior, and so new techniques are required. 

What follows is the main idea of this paper:
Draw a small subset $S\subseteq T$ at random and put $V:=\operatorname{val}(S\text{-}\mathrm{SDP})$.
Then $V$ satisfies (i)--(iii) above, and so we may compute the test statistic \eqref{eq.test statistic} to conclude a high-confidence lower bound on $\operatorname{val}(T\text{-}\mathrm{IP})$.
Indeed, (i) follows from $S$ being small, (ii) is obvious, whereas (iii) is less trivial:

\begin{theorem}
Pick $s\leq|T|$ and draw $S$ uniformly from $\binom{T}{s}$.
Then $\mathbb{E}\operatorname{val}(S\text{-}\mathrm{SDP})\leq\operatorname{val}(T\text{-}\mathrm{IP})$.
\end{theorem}

\begin{proof}
Let $\{C_t\}_{t\in[k]}$ denote a minimizer of $T$-IP.
Then $\{C_t\cap S\}_{t\in [k]}$ is feasible in ($S$-IP), and so
\begin{align}
\mathbb{E}\operatorname{val}(S\text{-}\mathrm{IP})
\nonumber&\leq\mathbb{E}\left[\frac{1}{s}\sum_{t\in[k]}\sum_{i\in C_t\cap S}\bigg\|x_i-\frac{1}{|C_t\cap S|}\sum_{j\in C_t\cap S}x_j\bigg\|^2\right]\\
\label{eq.centroid bound}&\leq\frac{1}{s}\cdot\mathbb{E}\left[\sum_{t\in[k]}\sum_{i\in C_t\cap S}\bigg\|x_i-\frac{1}{|C_t|}\sum_{j\in C_t}x_j\bigg\|^2\right]\\
\nonumber&=\frac{1}{s}\cdot\frac{1}{\binom{|T|}{s}}\sum_{S\in\binom{T}{s}}\sum_{t\in[k]}\sum_{i\in C_t\cap S}\bigg\|x_i-\frac{1}{|C_t|}\sum_{j\in C_t}x_j\bigg\|^2\\
\nonumber&=\frac{1}{s}\cdot\frac{\binom{|T|-1}{s-1}}{\binom{|T|}{s}}\sum_{t\in[k]}\sum_{i\in C_t}\bigg\|x_i-\frac{1}{|C_t|}\sum_{j\in C_t}x_j\bigg\|^2\\
\nonumber&=\operatorname{val}(T\text{-}\mathrm{IP}),
\end{align}
where \eqref{eq.centroid bound} follows from the fact that the centroid of points $\{x_i\}_{i\in C_t\cap S}$ minimizes the sum of squared distances from those points.
Next, every feasible point of ($S$-IP) corresponds to a feasible point of ($S$-SDP) with the same value, and so $\operatorname{val}(S\text{-}\mathrm{SDP})\leq\operatorname{val}(S\text{-}\mathrm{IP})$.
The result then follows from taking the expectation and combining with the previous bound.
\end{proof}

Furthermore, this lower bound provides a substantial improvement over the k-means++ guarantee.
Case in point, for the MNIST training set, this choice of $V$ tends to be around 37 when $|S|=450$.
Figure~\ref{fig.lower_bound} illustrates additional numerical experiments along these lines.

We conclude this section with a theoretical guarantee that our method yields a 3-approximation certificate when the data comes from a mixture of Gaussians model.
The constant 3 is not tight; rather, we selected this constant to make the proof clean, which appears in Section~4.

\begin{theorem}
\label{thm.gmm}
Take $|T|\geq m\log m$, $s=m\log m$, $B=(m+3)/3$ and $\ell=7$.
Draw $\{x_i\}_{i\in T}\subseteq\mathbb{R}^m$ independently from an arbitrary mixture of two Gaussians, each with identity covariance.
Denote the random variable $V=\operatorname{val}(S\text{-}\mathrm{SDP})$ for $S$ drawn uniformly from $\binom{T}{s}$.
Then with probability $1-O(1/m)$, the following occur simultaneously:
\begin{itemize}
\item[(i)]
$\operatorname{val}(T\text{-}\mathrm{IP})\leq 3B$, and
\item[(ii)]
$\operatorname{val}(T\text{-}\mathrm{IP})>B$ follows from the test statistic~\eqref{eq.test statistic} with 99\% confidence.
\end{itemize}
\end{theorem}

Perhaps the most important part of Theorem~\ref{thm.gmm} is that it suffices to solve the $k$-means SDP with only $s=m\log m$, a considerable savings in runtime when $T$ is large.

\section{Discussion}

This paper provides a sub-linear Monte Carlo approximation certificate that performs well on real-world data.
Additional improvements would likely come from more information about the distribution of $\operatorname{val}(S\text{-}\mathrm{SDP})$.
For example, if we had an almost-sure upper bound on $\operatorname{val}(S\text{-}\mathrm{SDP})$, then Hoeffding's inequality would enable us to use the sample average as a test statistic instead of the minimum.
According to our experiments on the MNIST data set, the sample average of $\operatorname{val}(S\text{-}\mathrm{SDP})$ is within a fraction of being optimal (see Figure~\ref{fig.lower_bound}), and so we expect tighter lower bounds to hold accordingly.
We suspect that improvements of this sort could arise by first computing geometric properties of the data set; we leave this analysis for future work. 

\begin{figure}
\centering
\includegraphics[height=.4\textwidth]{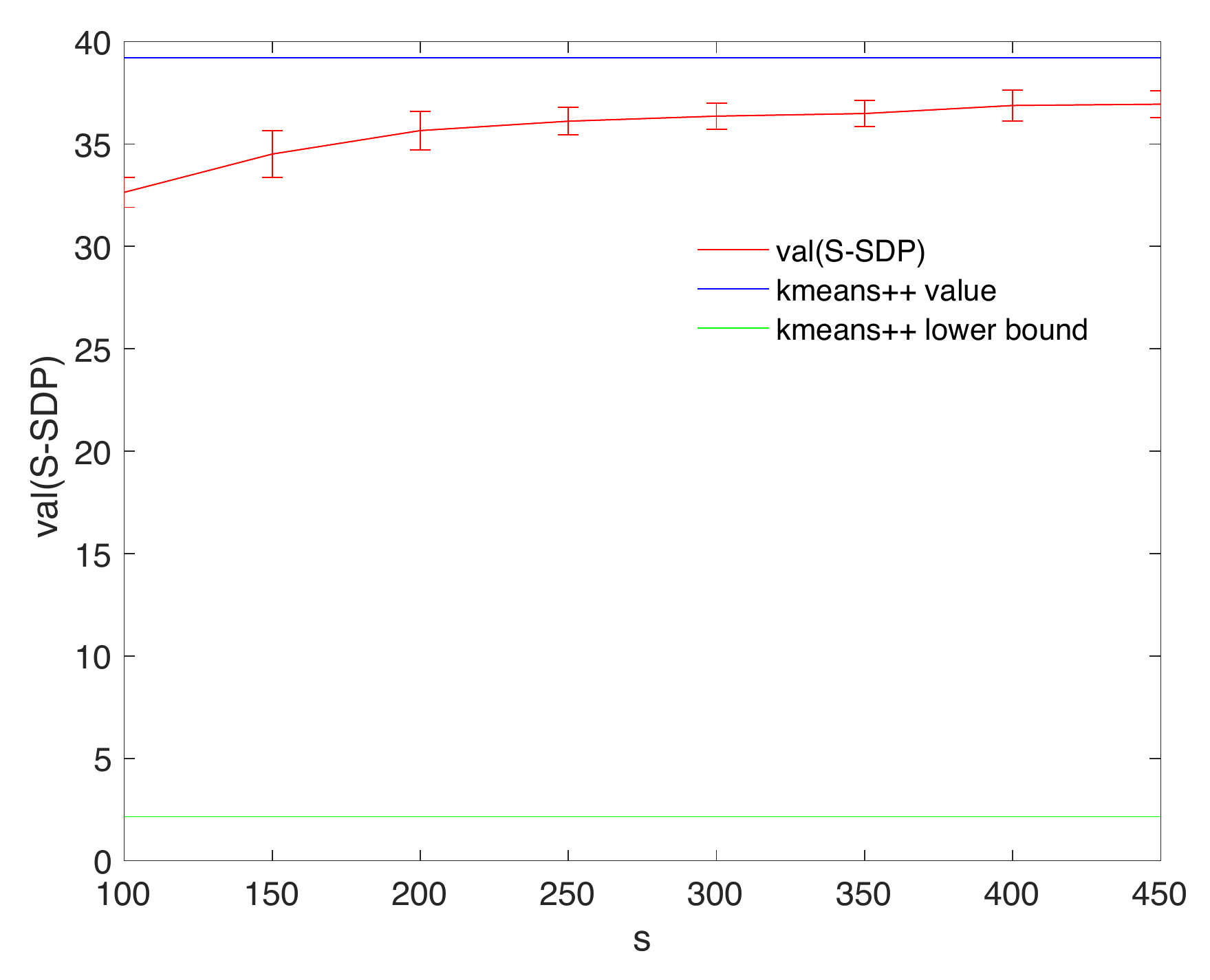}
\qquad\quad
\includegraphics[height=.4\textwidth]{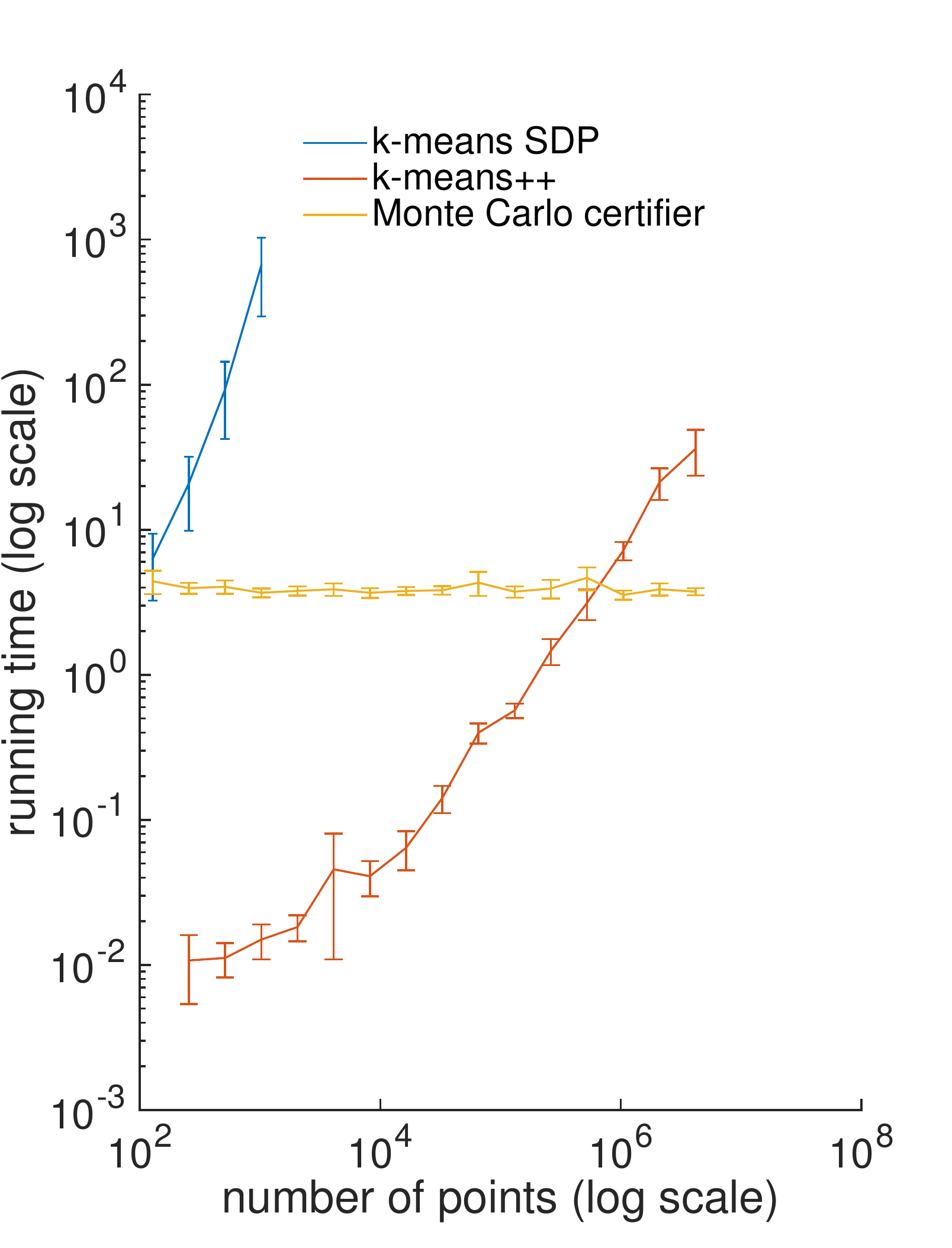}
\caption{\label{fig.lower_bound}
\textbf{(left)}
The $k$-means clustering value of all 60,000 handwritten digits from the MNIST training set (\citet{LeCunCB:online}).
The horizontal line at 39.2160 is an upper bound for the $k$-means value $\operatorname{val}(T\text{-}\mathrm{IP})$ obtained by running $k$-means++ with $k=10$, whereas the horizontal line at 2.1512 is the lower bound given by \eqref{eq.kmeans++}.
Between these lines, we plot the average of $\operatorname{val}(S\text{-}\mathrm{SDP})$ over 10 samples as a function of $|S|=s$.
Here, the error bars correspond to one standard deviation. 
\textbf{(right)}
Running times for $k$-means clustering algorithms as a function of number of data points.
Draw different numbers $N$ of points from a fixed mixture of two Gaussians in $\mathbb R^4$.
As expected, MATLAB's built-in implementation of $k$-means++ exhibits linear complexity in $N$, whereas the running time of the $k$-means SDP is much slower.
By contrast, the running time of our Monte Carlo certifier is independent of the number of points, and each trial produced a 2-approximation certificate with confidence ranging from $97.2\%$ to $99.7\%$ (here, we used $\ell=11$).
In fact, our certifier was faster than $k$-means++ when clustering over one million points. The SDPs were solved using SDPNAL+v0.5 (\citet{SDPNAL+}).
}
\end{figure}

\section{Proof of Theorem~\ref{thm.gmm}}

Put $N=|T|$.
Without loss of generality, we may assume the Gaussians are centered at $rv$ and $-rv$ for some $r\geq0$; here, $v$ denotes the first identity basis element.
Then for some $n$, we have $x_i=rv+g_i$ for $i\in\{1,\ldots,n\}$ and $x_i=-rv+g_i$ for $i\in\{n+1,\ldots,N\}$, where the $g_i$'s are independent with standard Gaussian entries.
Then the planted clustering $C_1=\{1,\ldots,n\}$, $C_2=\{n+1,\ldots,N\}$ is feasible in $(T\text{-}\mathrm{IP})$, and so
\begin{equation}
\label{eq.val ip est}
\operatorname{val}(T\text{-}\mathrm{IP})
\leq\frac{1}{N}\sum_{t\in[2]}\sum_{i\in C_t}\bigg\|x_i-\frac{1}{|C_t|}\sum_{j\in C_t}x_j\bigg\|^2
\leq\frac{1}{N}\sum_{i\in T}\|g_i\|^2,
\end{equation}
where the second inequality follows from the fact that the centroid of points $\{x_i\}_{i\in C_t}$ minimizes the sum of squared distances from those points.

To estimate $\operatorname{val}(S\text{-}\mathrm{SDP})$, let $\mu$ denote the $S\times1$ vector whose $i$th entry is $\|x_i\|^2$, and let $y$ denote the $S\times1$ vector whose $i$th entry is $\langle x_i,v\rangle$.
Then it is straightforward to verify that
\[
D_S
=\mu1^\top+1\mu^\top-2\Big(yy^\top+G^\top(I-vv^\top)G\Big),
\]
where $G$ is the $m\times S$ matrix whose $i$th column is $g_i$.
Put
\[
\mathcal{X}=\Big\{X:X1=1,\operatorname{tr}(X)=k,X\geq0,X\succeq0\Big\}.
\]
Then
\begin{equation}
\label{eq.est sdp val 1}
\operatorname{val}(S\text{-}\mathrm{SDP})
=\min_{X\in\mathcal{X}}\frac{1}{2s}\operatorname{tr}(D_SX)
=\frac{1}{s}1^\top\mu-\frac{1}{s}\max_{X\in\mathcal{X}}\operatorname{tr}\Big(\big(yy^\top+G^\top(I-vv^\top)G\big)X\Big).
\end{equation}
This motivates the following definition:
\begin{equation}
\label{eq.F}
\mathcal{F}(M)
:=
\max_{X\in\mathcal{X}}
|\operatorname{tr}(MX)|
\end{equation}
Then $\mathcal{F}$ is a seminorm on the space of real symmetric matrices (i.e., it satisfies all of the usual norm properties, though not necessarily non-degeneracy).
We can estimate this seminorm in terms of more familiar quantities (cf.\ Lemma~9 in~\citet*{MixonVW:17}):

\begin{lemma}
Let $\|\cdot\|_*$ and $\|\cdot\|_{2\rightarrow2}$ denote nuclear and spectral norms, respectively.
Then
\[
\mathcal{F}(M)\leq\min\Big\{\|M\|_*,k\|M\|_{2\rightarrow 2}\Big\}.
\]
\end{lemma}

\begin{proof}
Von Neumann's trace inequality gives that
\[
|\operatorname{tr}(MX)|
\leq\sum_{i\in[s]}\alpha_i\beta_i,
\]
where $\alpha_1\geq\cdots\geq\alpha_N$ and $\beta_1\geq\cdots\geq\beta_N$ are the singular values of $M$ and $X$, respectively.
Since $X$ is feasible in \eqref{eq.F}, $X$ is necessarily stochastic, and so $\beta_1=1$.
From this we obtain the first bound:
\[
|\operatorname{tr}(MX)|
\leq\sum_{i\in[s]}\alpha_i
=\|M\|_*.
\]
Since we additionally have $\sum_{i\in[s]}\beta_i=\operatorname{tr}(X)=k$, we may also conclude
\[
|\operatorname{tr}(MX)|
\leq\sum_{i\in[k]}\alpha_i
\leq k\|M\|_{2\rightarrow2}.
\qedhere
\]
\end{proof}

This lemma then gives
\begin{align*}
\mathcal{F}\big(yy^\top+G^\top(I-vv^\top)G\big)
&\leq\mathcal{F}(yy^\top)+\mathcal{F}\big(G^\top(I-vv^\top)G\big)\\
&\leq\|yy^\top\|_*+k\big\|G^\top(I-vv^\top)G\big\|_{2\rightarrow2}
\leq\|y\|^2+k\|G^\top G\|_{2\rightarrow2}.
\end{align*}
We may combine with \eqref{eq.est sdp val 1} to get
\begin{align}
\operatorname{val}(S\text{-}\mathrm{SDP})
\nonumber&=\frac{1}{s}1^\top\mu-\frac{1}{s}\max_{X\in\mathcal{X}}\operatorname{tr}\Big(\big(yy^\top+G^\top(I-vv^\top)G\big)X\Big)\\
\nonumber&\geq\frac{1}{s}1^\top\mu-\frac{1}{s}\Big(\|y\|^2+k\|G\|_{2\rightarrow2}^2\Big)\\
\label{eq.final est}&=\frac{1}{s}\sum_{i\in[s]}\|(I-vv^\top)x_i\|^2-\frac{2}{s}\|G\|_{2\rightarrow2}^2.
\end{align}
To conclude the argument, we estimate \eqref{eq.val ip est} and \eqref{eq.final est} with the help of well-known deviation bounds:

\begin{proposition}[Lemma~1 in~\citet{LaurentM:00}]
\label{prop.chi}
Suppose $Q$ has chi-squared distribution with $n$ degrees of freedom.
Then
\[
\mathbb{P}\Big(Q\geq n+2\sqrt{nt}+2t\Big)\leq e^{-t}
\quad
\text{and}
\quad
\mathbb{P}\Big(Q\leq n-2\sqrt{nt}\Big)\leq e^{-t}
\qquad
\forall t\geq0.
\]
\end{proposition}

\begin{proposition}[Corollary~5.35 in~\citet{Vershynin:11}]
\label{prop.gauss spec}
Let $G$ be an $m\times s$ matrix whose entries are independent standard Gaussian random variables.
Then $\|G\|_{2\rightarrow2}\leq\sqrt{s}+\sqrt{m}+t$ with probability $\geq1-2e^{-t^2/2}$.
\end{proposition}

The right-hand side of \eqref{eq.val ip est} is chi-squared distributed with $mN$ degrees of freedom.
As such, taking $t=N/m\geq \log m$ in the first part of Proposition~\ref{prop.chi} gives
\[
\operatorname{val}(T\text{-}\mathrm{IP})
\leq m+2+2/m
\leq3B
\]
with probability $1-O(1/m)$.
Next, the sum in \eqref{eq.final est} is chi-squared distributed with $s(m-1)$ degrees of freedom.
As such, take $t=s/(m-1)\geq \log m$ in the second part of Proposition~\ref{prop.chi} and $t=\sqrt{2s/m}$ in Proposition~\ref{prop.gauss spec} to get
\[
\operatorname{val}(S\text{-}\mathrm{SDP})
=\frac{1}{s}\sum_{i\in[s]}\|(I-vv^\top)x_i\|^2-\frac{2}{s}\|G\|_{2\rightarrow2}^2
\geq m-5-O(\tfrac{1}{\sqrt{\log m}})
\geq 2B
\]
with probability $1-O(1/m)$.
Apply the union bound over all $\ell=7$ instances of $\operatorname{val}(S\text{-}\mathrm{SDP})$ to get $T\geq2B$ with probability $1-O(1/m)$.
In this event, we may reject $H_0$ with $p$-value
\[
(B/T)^\ell\leq(1/2)^7<0.01,
\]
as claimed.

\section*{Acknowledgments}
This research was conducted while S.V.\ was a Research Fellow at the Simons Institute for Computing, University of California at Berkeley. 
D.G.M.\ was partially supported by AFOSR F4FGA06060J007 and AFOSR Young Investigator Research Program award F4FGA06088J001.
S.V.\ was partially supported by the Simons Algorithms and Geometry (A$\&$G) Think Tank. 
The views expressed in this article are those of the authors and do not reflect the official policy or position of the United States Air Force, Department of Defense, or the U.S.\ Government.

\smallskip

The authors dedicate this work in memory of Michael Cohen.

\bibliographystyle{plainnat}
\bibliography{literature_review}

\end{document}